\documentclass{article}

\PassOptionsToPackage{numbers}{natbib}
\usepackage[final]{neurips_2024}

\usepackage{graphicx} 
\usepackage{amsmath}
\usepackage{amsfonts}
\usepackage{url}
\usepackage{microtype}
\usepackage{graphicx}
\usepackage{booktabs} 
\usepackage{appendix}
\usepackage{multirow}
\usepackage{amsthm}
\usepackage{algorithm}
\usepackage{algpseudocode}
\usepackage{amssymb}
\usepackage[capitalize,noabbrev]{cleveref}
\newtheorem{theorem}{Theorem}[section]

\newtheorem{proposition}[theorem]{Proposition}
\newtheorem{definition}[theorem]{Definition}

\title{
A General Recipe for Contractive Graph Neural Networks - Technical Report
}

\author{%
  Maya Bechler-Speicher \\
  Blavatnik School of Computer Science\\
 Tel-Aviv University \\
  \And
  Moshe Eliasof \\
 Department of Applied Mathematics and Theoretical Physics  \\
University of Cambridge \\
}
\begin{document}

\maketitle
\begin{abstract}
Graph Neural Networks (GNNs) have gained significant popularity for learning representations of graph-structured data due to their expressive power and scalability. However, despite their success in domains such as social network analysis, recommendation systems, and bioinformatics, GNNs often face challenges related to stability, generalization, and robustness to noise and adversarial attacks. Regularization techniques have shown promise in addressing these challenges by controlling model complexity and improving robustness. Building on recent advancements in contractive GNN architectures, this paper presents a novel method for inducing contractive behavior in any GNN through SVD regularization. By deriving a sufficient condition for contractiveness in the update step and applying constraints on network parameters, we demonstrate the impact of SVD regularization on the Lipschitz constant of GNNs. Our findings highlight the role of SVD regularization in enhancing the stability and generalization of GNNs, contributing to the development of more robust graph-based learning algorithms dynamics.
\end{abstract}

\section{Introduction}
\label{sec:intro}

In recent years, Graph Neural Networks (GNNs) have emerged as powerful tools for learning representations of graph-structured data, enabling breakthroughs across various fields, including social network analysis, recommendation systems, and bioinformatics \cite{wu2020comprehensive}. The unique ability of GNNs to capture the complex, non-Euclidean structure of graph data has contributed to their success. However, as the size and complexity of real-world datasets increase, GNNs often encounter challenges such as instability, overfitting, and vulnerability to adversarial attacks \cite{zhang2019spectral, scaman2018optimal}. These issues pose significant barriers to the scalability and generalization of GNNs, particularly in applications where data is large, noisy, or susceptible to manipulation.

Regularization techniques play a crucial role in addressing these limitations by imposing constraints that reduce the risk of overfitting and improve model interpretability. One such technique, singular value decomposition (SVD) regularization, has garnered attention for its ability to enforce structured constraints on weight matrices, thereby enhancing model robustness \cite{mahoney2011randomized}. SVD regularization not only mitigates overfitting but also promotes stability by controlling the Lipschitz constant of the network, which governs the sensitivity of the model to input perturbations.

Recently, \citet{eliasof2023contractive} introduced a contractive GNN architecture designed to improve robustness against adversarial attacks. By leveraging principles from contractive systems theory, they demonstrated how the inherent stability of contractive systems could be applied to GNNs to counteract adversarial influences. Building upon this work, we propose a general method to transform any GNN architecture into a contractive GNN. Specifically, we derive a sufficient condition for the update step in GNNs to be contractive and introduce constraints over network parameters that guarantee contractiveness.

Our approach leverages SVD regularization as a key mechanism for inducing contractive behavior, allowing us to stabilize GNNs without sacrificing their expressive power. By analyzing the effect of SVD regularization on the Lipschitz constant of GNNs, we provide new insights into how this technique can enhance both stability and generalization. This work contributes to the broader understanding of regularization techniques in GNNs and paves the way for future research aimed at improving the robustness and scalability of graph-based learning models.

\section{Imposing Contractivity on Popular GNNs}
\label{sec:method}
We now mathematically derive contractivity conditions for two popular GNNs, namely, GCN \cite{kipf2016semi} and GraphConv \cite{morris2019weisfeiler}. We choose to focus on these two architectures because they are well-known and many subsequent GNNs rely on them, such as GCNII \cite{chen20simple} and EGNN \cite{zhou2021dirichlet}. Furthermore, GraphConv is maximally expressive (1-WL) in terms of linear complexity MPNNs \cite{morris2019weisfeiler}.

\paragraph{Notations.}
We define a graph $G=(V,E)$ by its set of nodes $V$ and edges $E$. The graph can also be described by its adjacency matrix $A$ such that $A_{ij} = 1$ if the edge $(i,j)$ exists, i.e., $(i,j) \in E$. Otherwise, $A_{ij} = 0$. We denote the node features at the $l$-th layer by $X^{(l)} \in \mathbb{R}^{|V| \times C}$ where $C$ is the hidden feature dimension. Because our focus is on the contractivity of GNN models, we also denote a perturbation of the node features by $X^* = X + \epsilon$, such that $\epsilon \in \mathbb{R}^{|V| \times C}$ is a noise matrix. We denote an activation function (e.g., ReLU) by $\sigma(\cdot)$.

\begin{definition}[Contractive GNN Layer]\label{def:contractivity}
    A GNN layer $f()$, combined with a 1-Lipschitz activation function $\sigma$, is said to be contracting with respect to a graph $A$ if, for any pair of node feature matrices $X$ and $X^*$, the following inequality holds:
\begin{equation*}
\|\sigma(f(X; A)) - \sigma(f(X^*; A))\| \leq \|X - X^*\|
\end{equation*}

This definition implies that the distance between the transformed feature representations should not increase through the layer. In other words, the GNN layer should contract the perturbations or differences in the input feature matrices.
\end{definition}

\paragraph{GCN} The GCN model \cite{kipf2016semi} is described by the following update equation:
\begin{equation*}
    \label{eq:gcn_model}
    X^{(l+1)} = \sigma(AX^{(l)}W^{(l)})
\end{equation*}
where $W^{(l)} \in \mathbb{R}^{C \times C}$ are trainable weights.

We now derive the mathematical conditions with respect to Equation \cref{eq:gcn_model}, such that they impose contractivity on the GCN model. For simplicity, we will omit the layer notations $l$.

\begin{proposition}[Contractivity conditions for GCN]\label{proof:gcn}
Contractivity can be imposed on GCN by requiring:
        $$ \|W\| \leq \frac{1}{\|A\|}. $$
\end{proposition}

\begin{proof}
        
Let us assume that $\sigma$ is 1-Lipschitz (e.g., ReLU, LeakyReLU, Tanh), i.e. 
$$
\|\sigma(AXW) - \sigma(AX^*W)\| \leq \|AXW - AX^*W\|.
$$ 
Then, we can consider the linear part of the update equation of GCN, i.e., $X = AXW$. To obtain contractivity with respect to perturbed node features $X^*$ in this model, we need to satisfy the following:
\begin{equation}
    \label{eq:contractivity_demand_gcn}
    \|AXW - AX^*W\| \leq \|X - X^*\|.
\end{equation}

Let us denote the noise matrix by $ \epsilon = X - X^*$, then the following holds: $\|AXW - AX^*W\| = \|A(X - X^*)W\| = \|A \epsilon W\|$. Recall that according to \cref{eq:contractivity_demand_gcn} we need to require that:
\begin{align}
    \label{eq:derive_gcn_constraints}
    \|A \epsilon W \| \leq \|\epsilon\|.
\end{align}

We can rephrase \cref{eq:derive_gcn_constraints} as follows:
\begin{align*}
    & \|A \epsilon W \| \leq \|\epsilon\| \\
    & \iff \| \text{vec}(A\epsilon W) \| \leq \| \text{vec}(\epsilon) \| \\ 
    & \iff \| (W \otimes A) \text{vec}(\epsilon) \| \leq \| \text{vec}(\epsilon) \|,
\end{align*}
where $\otimes$ denotes the Kronecker product.

By using the sub-multiplicative property of norms, we can bound the above inequality as follows:
$$ \|(W \otimes A) \text{vec}(\epsilon) \| \leq \|W \otimes A\| \| \text{vec}(\epsilon) \|,$$
and therefore to obtain contractivity we demand:
$$
\|W \otimes A\| \| \text{vec}(\epsilon) \| \leq \| \text{vec}(\epsilon) \|.
$$
Note that this demand is equivalent to the following inequality $\|W \otimes A\| \leq 1$. Using the property of Kronecker products, we get that:
$$ \|W \otimes A\| = \|W\| \|A\|.$$
Therefore, \cref{eq:contractivity_demand_gcn} can be satisfied by demanding: 
$$ \|W\| \leq \frac{1}{\|A\|}.$$
\end{proof}

\paragraph{GraphConv}  The GraphConv model \cite{morris2019weisfeiler} admits to the following update equation:

\begin{equation*}
    \label{eq:graphconv}
    X^{(l+1)} = \sigma(X^{(l)}W_1^{(l)} + \tilde{A}X^{(l)}W_2^{(l)})
\end{equation*}
where $W_1, W_2 \in \mathbb{R}^{C \times C}$ are learnable weight matrices, and $\tilde{A}$ is the adjacency matrix of the graph without self-loops. We now derive the mathematical conditions with respect to Equation \cref{eq:graphconv}, such that they impose contractivity on the GraphConv model. Similar to the case of GCN, and for simplicity, we will omit the layer notations $l$. 

\begin{proposition}[Contractivity conditions for GraphConv]
    \label{prop:conractGraphConv}
    A GraphConv layer is contractive if:
    $$
      \|W_1^\top \| + \| W_2^\top \| \| \tilde{A} \| \leq 1
    $$
\end{proposition}

\begin{proof}
    Let us assume that $\sigma$ is 1-Lipschitz (e.g., ReLU, LeakyReLU, Tanh), i.e. 
\begin{align} \nonumber
\|  \sigma(XW_1 + \tilde{A}XW_2) - \sigma( X^*W_1 + \tilde{A}X^*W_2 ) \| \\ \nonumber \leq \| XW_1 + \tilde{A}XW_2 - ( X^*W_1 + \tilde{A}X^*W_2 ) \|.
\end{align} Then we can consider the linear part of  GraphConv layer that reads $X = XW_1 + \tilde{A}XW_2$. To obtain contractivity we want to have the following:
\begin{equation}
\label{eq:condition_graphconv}
    \| XW_1 + \tilde{A}XW_2 - ( X^*W_1 + \tilde{A}X^*W_2 ) \| \leq \| X - X^* \|
\end{equation}
By expanding the left hand side of \cref{eq:condition_graphconv}, we obtain:
\begin{align}
    &\notag \| XW_1 + \tilde{A}XW_2 - (X^*W_1 + \tilde{A}X^*W_2)  \| \\ &\notag= \| (X-X^*)W_1 + \tilde{A}(X-X^*)W_2  \| \\
    &\notag = \| I(X-X^*)W_1 + \tilde{A}(X-X^*)W_2 \| \\ \label{eq:expansion_condition_graphconv}
    &= \| I \epsilon W_1 + \tilde{A} \epsilon W_2 \|,
\end{align}
where $\epsilon = (X-X^*)$.
 Recall the following identity of the Kronecker product: $${\rm{vec}}(AXB^T) = (B \otimes A) {\rm{vec}(X)},$$ which holds for any triplet of matrices $A,X,B$.

By considering the norm of the vectorization \cref{eq:expansion_condition_graphconv}, we obtain the following expression:

\begin{align}
    &\notag \| {\rm{vec}}(I \epsilon W_1) + {\rm{vec}}(\tilde{A} \epsilon W_1)  \| \\
    &\notag = \| (W_1^\top \otimes I){\rm{vec}}(\epsilon) + (W_2^\top \otimes \tilde{A}){\rm{vec}}(\epsilon)  \| \\
    &\notag  \leq \|  (W_1^\top \otimes I){\rm{vec}}(\epsilon) \| + \|  (W_2^\top \otimes \tilde{A}){\rm{vec}}(\epsilon)  \| \\ \label{eq:norm}
    &\leq  \|  (W_1^\top \otimes I) \| \| {\rm{vec}}(\epsilon) \| + \| (W_2^\top \otimes \tilde{A}) \| \| {\rm{vec}}(\epsilon) \| 
\end{align}

Recall that in order to obtain contractivity we demand that:
\begin{equation}
\label{eq:contractivity_constraint}
 \| {\rm{vec}}(I \epsilon W_1) + {\rm{vec}}(\tilde{A} \epsilon W_1)  \| \leq \|{\rm{vec}}(\epsilon)\|.  
 \end{equation}

Therefore, combining  \cref{eq:norm} and \cref{eq:contractivity_constraint}  we  demand that:
\begin{align*}
     & \|  (W_1^\top \otimes I) \| \| {\rm{vec}}(\epsilon) \| + \| (W_2^\top \otimes \tilde{A}) \| \| {\rm{vec}}(\epsilon) \| \leq \|{\rm{vec}}(\epsilon)\|
\end{align*}
Therefore, we need to require that $\| ((W_1^\top \otimes I))\| + \|(W_2^\top \otimes \tilde{A})) \| \leq 1.$

It can be proven \footnote{\url{https://math.stackexchange.com/questions/2342156/matrix-norm-of-kronecker-product}} that for the $\ell_2$ induced matrix norm, it holds that $\|A \otimes B \| = \|A\| \|B\|$. Thus we need to require the following:

\begin{align*}
    \| ((W_1^\top \otimes I)) + (W_2^\top \otimes \tilde{A})) \| &\leq \|W_1^\top \| \|I\| + \| W_2^\top \| \| \tilde{A} \| \\ &= \|W_1^\top \| + \| W_2^\top \| \| \tilde{A} \|  \leq 1
\end{align*}

\end{proof}

\section{Implementing Contractivity}
\label{sec:implementing}
In Propositions \ref{proof:gcn} and \ref{prop:conractGraphConv}, we presented conditions to obtain contractivity for GCN and GraphConv, respectively. We now show that these conditions can be satisfied using SVD regularization on the learned weight matrices of the models. This implementation highlights the interesting connection between the well-known SVD regularization \cite{van1996matrix}, and contractivity in GNNs. 
We first present a proposition regarding the ability of SVD decomposition to implement contractivity for GCN and GraphConv layers, with its proof provided in the Appendix.

\begin{proposition}[SVD Implements Contractivity]
\label{prop:svd_contractivity}
Given a weight matrix $W$ and a threshold $\tau$, the SVD decomposition of $W$ and the subsequent modification of its singular values according to the threshold $\tau$ ensures that the resulting weight matrix $\tilde{W}$ satisfies the contractivity condition $\|W\| \leq \tau$.
\end{proposition}

\begin{proof}
    Let $W \in \mathbb{R}^{m \times n}$ be a weight matrix. The SVD decomposition of $W$ is given by:
\begin{equation*}
W = U S V^\top,
\end{equation*}
where $U \in \mathbb{R}^{m \times m}$ and $V \in \mathbb{R}^{n \times n}$ are orthogonal matrices, and $S \in \mathbb{R}^{m \times n}$ is a diagonal matrix containing the singular values $\{s_i\}_{i=1}^{\min(m,n)}$ of $W$.

To enforce the contractivity condition $\|W\| \leq \tau$, we modify the singular values of $S$ as follows:
\begin{equation*}
\tilde{S}_{ii} = \begin{cases}
s_{ii}, \quad \text{if } s_{ii} \leq \tau, \\
\tau, \quad \text{otherwise}.
\end{cases}
\end{equation*}

Let $\tilde{W}$ be the modified weight matrix:
\begin{equation*}
\tilde{W} = U \tilde{S} V^\top.
\end{equation*}

We need to show that $\|\tilde{W}\| \leq \tau$. Recall that for the induced $\ell_2$ matrix norm (also known as the spectral norm), we have:
\begin{equation*}
\|W\|_2 = \max_i s_i,
\end{equation*}
where $s_i$ are the singular values of $W$.

By construction, the singular values of $\tilde{W}$, denoted as $\tilde{s}_i$, are:
\begin{equation*}
\tilde{s}_i = \min(s_i, \tau).
\end{equation*}

Thus, the largest singular value of $\tilde{W}$ is:
\begin{equation*}
\max_i \tilde{s}_i \leq \tau.
\end{equation*}

Therefore, we have:
\begin{equation*}
\|\tilde{W}\|_2 = \max_i \tilde{s}_i \leq \tau,
\end{equation*}
which satisfies the contractivity condition.

\end{proof}
\paragraph{Contractive GCN} Recall that according to Proposition \ref{proof:gcn}, we need to satisfy the following inequality for a GCN layer to be contractive:
 $$ \|W\| \leq \frac{1}{\|A\|}$$ 
In this case, the adjacency matrix $A$ is fixed (given as an input), and therefore $\|A\|$ is known, while $W$ is learned. We can enforce contractivity in GCN by modifying the learned weights $W$ as follows:

\begin{align*}
    \label{eq:svd_gcn}
    U, S, V &= {\rm{SVD}}(W), \\
    \tilde{S}_{ii} &= \begin{cases}
        S_{ii}, \quad \text{if } S_{ii} \leq \frac{1}{\|A\|} \\
        \frac{1}{\|A\|}, \quad \text{otherwise}
    \end{cases} \\
    \tilde{W} &= U \tilde{S} V^\top
\end{align*}
where $\text{SVD}$ denotes the singular value decomposition, and $S$ is a diagonal matrix of the singular values of $W$.
Therefore, a contractive GCN layer reads:
\begin{equation*}
    \label{eq:gcn_contractive_model}
    X^{(l+1)} = \sigma(AX^{(l)}\tilde{W}^{(l)}).
\end{equation*}

\paragraph{Contractive GraphConv.} As presented in Proposition \ref{prop:conractGraphConv}, GraphConv is contractive if the following condition is met:
 
 $$
  \|W_1^\top \| + \| W_2^\top \| \| \tilde{A} \| \leq 1
$$

We now follow a similar update to the learned weights of GraphConv. The main difference here compared to the case of GCN is that we have two terms that condition the contractivity of the model. Since we do not have a specific preference on which term is to be bounded, we introduce a coefficient $\alpha \in [0,1]$ such that the (weighted, but equivalent in terms of contractivity) condition is:
\begin{equation}
    \label{eq:weightedConditiongconv}
      \alpha \|W_1^\top \| + (1-\alpha)\| W_2^\top \| \| \tilde{A} \| \leq 1
\end{equation}
 
To satisfy \cref{eq:weightedConditiongconv}, we demand the following:
\begin{align}
          \|W_1^\top \|  &\leq \alpha, \label{eq:demandW1}  \\
          \| W_2^\top \|&\leq \frac{1-\alpha}{\|\tilde{A}\|} \label{eq:demandW2}
\end{align}
Using the SVD decomposition, we can satisfy \cref{eq:demandW1} as follows:
\begin{align*}
    U, S, V &= {\rm{SVD}}(W_1) \\
    \tilde{S}_{ii} &= \begin{cases}
        S_{ii}, \quad \text{if } S_{ii} < \alpha \\
        \alpha, \quad \text{otherwise}
    \end{cases} \\
    \tilde{W_1} &= U \tilde{S} V^\top
\end{align*}
Similarly, \cref{eq:demandW2} can be satisfied as follows:
\begin{align*}
    U, S, V &= {\rm{SVD}}(W_2) \\
    \tilde{S}_{ii} &= \begin{cases}
        S_{ii}, \quad \text{if } S_{ii} < \frac{1-\alpha}{\|\tilde{A}\|} \\
        \frac{1-\alpha}{\|\tilde{A}\|} \quad \text{otherwise}
    \end{cases} \\
    \tilde{W_2} &= U \tilde{S} V^\top
\end{align*}
leading to the overall contractive update rule of GraphConv:

\begin{equation*}
    \label{eq:graphconv_contractive}
    X^{(l+1)} = \sigma(X^{(l)}\tilde{W}_1^{(l)} + \tilde{A}X^{(l)}\tilde{W}_2^{(l)})
\end{equation*}

\section{Summary}
This paper introduces a method to make any GNN architecture contractive by applying SVD regularization, enhancing stability and preventing overfitting. By analyzing its impact on the Lipschitz constant, the approach improves both the robustness and scalability of GNNs, building on recent advancements in contractive systems.

\bibliographystyle{unsrtnat}
\bibliography{biblio.bib}
\newpage
\end{document}